\documentclass{article}

\usepackage[preprint]{nips_2018_workshop}

\usepackage[utf8]{inputenc} 
\usepackage{amsmath,amstext,amssymb,amsopn,amsthm}
\usepackage{mathpazo}
\usepackage{bm}
\usepackage[T1]{fontenc}    
\usepackage{hyperref}       
\usepackage{url}            
\usepackage{booktabs}       
\usepackage{amsfonts}       
\usepackage{nicefrac}       
\usepackage{microtype}      
\usepackage{graphicx}
\usepackage{wrapfig}
\usepackage[font=small,labelfont=bf,textfont=it]{caption}
\usepackage{mathtools}
\usepackage{xcolor}
\usepackage{tikz}
\usepackage{adjustbox}
\usepackage[toc,page]{appendix} 
\usepackage{comment}
\usepackage{subfigure}
\newcommand{\C}{\mathbb{C}}
\newcommand{\N}{\mathbb{N}}
\newcommand{\R}{\mathbb{R}}

\newcommand{\M}{\mathcal{M}}

\newcommand{\W}{\mathcal{W}}

\newcommand{\Diff}{\mathrm{Diff}}
\newcommand{\diam}{\mathrm{diam}}

\newcommand{\Euc}{\mathrm{Euc}}
\newcommand{\Isom}{\mathrm{Isom}}
\newcommand{\id}{\mathrm{id}}

\newcommand{\vol}{\mathrm{vol}}

\newcommand{\Cb}{\mathbf{C}}

\newcommand{\Lb}{\mathbf{L}}

\newcommand{\ellb}{\bm{\ell}}

\newtheorem{theorem}{Theorem}[section]
\newtheorem{proposition}[theorem]{Proposition}
\newtheorem{lemma}[theorem]{Lemma}

\theoremstyle{definition}
\newtheorem{definition}{Definition}
\newtheorem{remark}[theorem]{Remark}

\title{Geometric Scattering on Manifolds}

\author{
Michael Perlmutter\thanks{Corresponding author} \\
 Dept. of CMSE\thanks{Computational Mathematics, Science \& Engineering}\\
 Michigan State University\\
 East Lansing, MI, USA\\
 \texttt{perlmut6@msu.edu}
\And
Guy Wolf\thanks{\href{http://guywolf.org/}{\texttt{http://guywolf.org/}}}\\
Dept. of Math. and Stat. \\
Universit\'{e} de Montr\'{e}al \\
  Montreal, QC, Canada \\
\texttt{guy.wolf@umontreal.edu}
\And
Matthew Hirn\thanks{\href{https://matthewhirn.com/}{\texttt{https://matthewhirn.com/}}} \\ 
Dept. of CMSE\footnotemark[2]\\
Dept. of Mathematics\\
Michigan State University\\
East Lansing, MI, USA \\
\texttt{mhirn@msu.edu}
}

\begin{document}

~\vspace{-50pt}

\maketitle

\vspace{-30pt}

\begin{abstract}%
The Euclidean scattering transform was introduced nearly a decade ago to improve the mathematical understanding of the success of convolutional neural networks (ConvNets) in image data analysis and other tasks. Inspired by recent interest in geometric deep learning, which aims to generalize ConvNets to manifold and graph-structured domains, we generalize the scattering transform to compact manifolds. Similar to the Euclidean scattering transform, our geometric scattering transform is based on a cascade of designed filters and pointwise nonlinearities, which enables rigorous analysis of the feature extraction provided by scattering layers. Our main focus here is on theoretical understanding of this geometric scattering network, while setting aside implementation aspects, although we remark that application of similar transforms to graph data analysis has been studied recently in related work. Our results establish conditions under which geometric scattering provides localized isometry invariant descriptions of manifold signals, and conditions in which the filters used in our network are stable to families of diffeomorphisms formulated in intrinsic manifolds terms. These results not only generalize the deformation stability and local roto-translation invariance of Euclidean scattering, but also demonstrate the importance of linking the used filter structures (e.g., in geometric deep learning) to the underlying manifold geometry, or the data geometry it represents.
\end{abstract}


\section{Introduction}

Characterizing variability in data is one of the most fundamental aspects of modern data analysis, which appears at the core of both supervised and unsupervised learning tasks. In particular, a crucial part of typical machine learning methods is to separate informative sources of variability from disruptive sources, which are considered via deformation groups or noise models. For example, linear classification methods aim to find a separating hyperplane that captures undesired (i.e., irrelevant for discriminating between classes) variance, and then eliminate it by projecting the data on normal directions of the hyperplane. Principal component analysis (PCA), on the other hand, aims to find a hyperplane that maximizes the captured variance, while treating directions of minimal variance as noise directions. Nonlinear methods extend such notions to consider richer models (beyond linear hyperplanes and projections) for capturing or eliminating variability - either task dependent or following an assumed model for the data. 

In the past decade, deep learning methods have shown impressive potential in capturing complex nonlinear variability by using a cascade of linear transformations and simple nonlinear activations that together form 
approximations of target functions. While there is a multitude of applications where such methods are effective, some of their most popular achievements are in fields traditionally associated with signal processing in general, and image processing (i.e., computer vision) in particular. In such settings, the data has an inherent spatial (or temporal) structure, and collected observations form signals over it. Convolutional neural networks (ConvNets) utilize this structure to treat their linear transformations as convolutions of input signals with filters that are learned in the training process. In classification applications, the resulting deep cascade of convolutional filters (and nonlinearities) eliminates intra-class heterogeneity and results in high accuracy classifiers. Furthermore, hidden ConvNet layers yield coefficients that isolate semantic features and can be utilized for unsupervised feature extraction, e.g., as part of dimensionality reduction or generative~models.

In an effort to improve mathematical understanding of deep convolutional networks and their learned features, \citet{mallat:firstScat2010,mallat:scattering2012} presented the \emph{scattering transform}. This transform has an architecture similar to ConvNets, based on a cascade of convolutional filters and simple pointwise nonlinearities in the form of complex modulus or absolute value. However, unlike deep learning methods, this transform does not learn its filters from data, but rather has them designed to provide guaranteed stability to a given family of deformations. As shown in~\citet{mallat:scattering2012}, under some admissibility conditions, one can use appropriate wavelet filter banks in the scattering transform to provide invariance to the actions of Lie groups. Moreover, the resulting scattering features also provide Lipschitz stability to small diffeomorphisms, where the size of a diffeomorphism is quantified by its deviation from a translation.
These notions were made concrete in \citet{bruna:scatClass2011,bruna:invariantScatConvNet2013},  \citet{sifre:rotoScatTexture2012,mallat:rotoScat2013,mallat:rigidMotionScat2014} and \citet{oyallon:scatObjectClass2014} using groups of translations, rotations, and scaling operations, with applications in image and texture classification. Further applications of the scattering transform and its deep filter bank approach were shown effective in several fields, such as audio processing \citep{anden:scatAudioClass2011,anden:deepScatSpectrum2014,wolf:BSS-mlsp,wolf:BSS, arXiv:1807.08869}, medical signal processing \citep{talmon:scatManifoldHeart2014}, and quantum chemistry~\citep{hirn:waveletScatQuantum2016, eickenberg:3DSolidHarmonicScat2017, eickenberg:scatMoleculesJCP2018, brumwell:steerableScatLiSi2018}.

Another structure often found in modern data is that of a graph between collected observations or features. Such structure naturally arises, for example, in transportation and social networks, where weighted edges represent connections between locations (e.g., roads or other routes) or user profiles (e.g., friendships or exchanged communications). They are also common when representing molecular structures in chemical data, or biological interactions in biomedical data. Relatedly, signals supported on manifolds or manifold valued data is also becoming increasing prevalent, in particular in shape matching and computer graphics. Manifold models for high dimensional data also arise in the field of manifold learning \citep[e.g.,][]{tenenbaum:isomap2000,coifman:diffusionMaps2006,maaten:tSNE2008}, in which unsupervised algorithms infer data-driven geometries and use them to capture intrinsic structure and patterns in data. 
As such, a large body of work has emerged to explore the generalization of spectral and signal processing notions to manifolds \citep[e.g.,][]{coifman:geometricHarmonics2006} and graphs \citep[and references therein]{shuman:emerging2013}. In these settings, functions are supported on the manifold or the vertices of the graph, and the eigenfunctions of the Laplace-Beltrami operator or the eigenvectors of the graph Laplacian serve as the Fourier harmonics. 


This increasing interest in non-Euclidean data, particularly graphs and manifolds,
has led to a new research direction known as \emph{geometric deep learning}, which
aims to generalize convolutional networks to graph and manifold structured data~\citep[and references therein]{Bronstein:geoDeepLearn2017}. Unlike classical ConvNets, in which filters are learned on collected data features (i.e., in spatial or temporal domain), many geometric deep learning approaches for manifolds learn spectral coefficients of their filters \citep{bruna:spectralNN2014, NIPS2016_6081, Levie:CayleyNets2017, yi:syncspecCNN2017}. The frequency spectrum is defined by the eigenvalues of the Laplace-Beltrami operator on the manifold, which links geometric deep learning with manifold and graph signal processing. 


Inspired by geometric deep learning, recent works have also proposed an extension of the scattering transform
to graph domains. These mostly focused on finding features that represent a graph structure (given a fixed set of signals on it) while being stable to graph perturbations. In \citet{gama:diffScatGraphs2018}, a cascade of diffusion wavelets from~\citet{coifman:diffWavelets2006} was proposed, and its Lipschitz stability was shown with respect to a global diffusion-inspired distance between graphs. A similar construction discussed in \citet{zou:graphCNNScat2018} was shown to be stable to permutations of vertex indices, and to small perturbations of edge weights. Finally, \citet{gao:graphScat2018} established the viability of scattering coefficients as universal graph features for data analysis tasks (e.g., in social networks and biochemistry data).


In this paper, we take the less charted path and consider the manifold aspect of geometric deep learning. In this setting, one needs to process signals over a manifold, and in particular, to represent them with features that are stable to orientations, noise, or deformations over the manifold geometry. 
In order to work towards these aims, we define a scattering transform on compact smooth Riemannian manifolds without boundary, which we call \emph{geometric scattering}. Our construction is based on convolutional filters defined spectrally via the eigendecomposition of the Laplace-Beltrami operator over the manifold, as discussed in Section~\ref{sec: ops}. We show that these convolutional operators can be used to construct a frame, which, with appropriately chosen low-pass and high-pass filters, forms a wavelet frame similar to the diffusion wavelets constructed in \citet{coifman:diffWavelets2006}. Then,  in Section~\ref{sec: transform}, a cascade of these generalized convolutions and pointwise absolute value operations is used to map signals on the manifold to scattering coefficients that encode approximate local invariance to isometries, which correspond to translations, rotations, and reflections in Euclidean space. In Section~\ref{sec: stability}, we analyze the commutators of the filters used in our network with the action of diffeomorphisms. We believe that the results presented there will allow us to study the stability of our scattering network to these diffeomorphisms in future work using a notion of stability that analogous to the Lipschitz stability considered in~\citet{mallat:scattering2012} on Euclidean space. 
As we discuss in Section~\ref{sec: stability}, this requires a quantitative notion of deformation size. We consider three formulations of such a notion, 
which measure how far a given diffeomorphism is from being an isometry, 
and explore the type of stability that can be achieved for the geometric scattering transform with respect to each of them. Our results provide a path forward for utilizing the scattering mathematical framework to analyze and understand geometric deep learning, while also shedding light on the challenges involved in such generalization to non-Euclidean domains.



\subsection{Notation}
Let $\M$ be a smooth, compact, and connected, $d$-dimensional Riemannian manifold without boundary contained in $\mathbb{R}^n$. 
Let $r:\M\times\M\rightarrow\mathbb{R}$ denote the geodesic distance between two points, and let
 $\Delta$ be the Laplace-Beltrami operator on $\M.$ The eigenfunctions and non-unique eigenvalues of $-\Delta$ are denoted by $\varphi_k$ and $\lambda_k$, respectively. Since $\M$ is compact, the spectrum of $-\Delta$ is countable and
we may assume that  $\{ \varphi_k \}_{k \in \N}$ forms an orthonormal basis for $\Lb^2 (\M)$. We define the Fourier transform of $f \in \Lb^2 (\M)$ as the sequence $\widehat{f} \in \ellb^2$ defined by $\widehat{f}(k) = \langle f, \varphi_k \rangle$. The set of unique eigenvalues of $-\Delta$ is denoted by $\Lambda$, and for $\lambda\in\Lambda$ we let $m(\lambda)$ and $E_\lambda$ denote the corresponding multiplicities and eigenspaces. The diffeomorphism group of $\M$ is $\Diff (\M)$, and the isometry group is $\Isom (\M) \subset \Diff (\M)$. For a diffeomorphism $\zeta \in \Diff (\M)$, we let $V_\zeta$ be the operator $V_\zeta f(x)=f(\zeta^{-1}(x)),$ and let $\|\zeta\|_\infty=\sup_{x\in\M} r(x,\zeta(x))$.  We  let $C(\M)$ denote a constant which depends only on the manifold $\M,$ and for two operators $A$ and $B$ defined on $\Lb^2(\M),$ we let $[A,B]\coloneqq AB-BA$ denote their commutator.

\section{Problem Setup}
\label{sec: problem setup}


We consider signals $f \in \Lb^2 (\M)$ and  representations $\Theta : \Lb^2 (\M) \rightarrow \ellb^2 (\Lb^2 (\M)),$
which encode information about the signal, 
often referred to as the ``features'' of $f$. A common goal in machine learning tasks is to classify or cluster signals based upon these features. However, 
 we often want to consider two signals, or even two manifolds, to be equivalent if they differ by the action of a global isometry. 
Therefore, we seek to construct a family of representations, $(\Theta_t)_{t \in (0, \infty)}$, which are invariant to isometric transformations of any $f \in \Lb^2 (\M)$ up to the scale $t$. Such a representation should satisfy a condition similar to, 
\begin{equation} \label{eqn: generic isometry invariance}
    \| \Theta_t (f) - \Theta_t (V_{\zeta} f) \|_{2,2} \leq \alpha (\zeta) \beta (t) \| f \|_2, \quad \forall \, f \in \Lb^2 (\M), \, \zeta \in \Isom (\M),
\end{equation}
where $\alpha (\zeta)$ measures the size of the isometry with $\alpha (\id) = 0$, and $\beta(t)$ decreases to zero as the scale $t$ grows to infinity. 

Along similar lines, it is desirable that the action of small diffeomorphisms on $f,$ or on the underlying manifold $\M,$ should not have a large impact on the representation of the inputted signal for tasks such as  classification or regression. However, the set of $\Diff(\M)$ is a large group and invariance over $\Diff (\M)$ would collapse the variability even between vastly different signals. Thus, in this case, we want a family $(\Theta_t)_{t \in (0, \infty)}$ that is stable to diffeomorphism actions on $f$, but not invariant. This leads to a condition such as:
\begin{equation} \label{eqn: generic diffeo stability}
    \| \Theta_t(f) - \Theta_t (V_{\zeta} f) \|_{2,2} \leq A(\zeta) \| f \|_2, \quad \forall \, t \in (0, \infty), \, f \in \Lb^2 (\M), \, \zeta \in \Diff (\M),
\end{equation}
where $A(\zeta)$ measures how much $\zeta$ differs from being an isometry, with $A(\zeta) = 0$ if $\zeta \in \Isom (\M)$ and $A(\zeta) > 0$ if $\zeta \notin \Isom (\M)$.

Combining \eqref{eqn: generic isometry invariance} and \eqref{eqn: generic diffeo stability}, we have the following goal:
\begin{align}
    \| \Theta_t(f) - \Theta_t (V_{\zeta} f) \|_{2,2} \leq [\alpha (\zeta) \beta (t) + &A(\zeta)] \| f \|_2, \label{eqn: ideal bound} \\
    &\forall \, t \in (0, \infty), \, f \in \Lb^2 (\M), \, \zeta \in \Diff (\M). \nonumber
\end{align}
At the same time, the representations $(\Theta_t)_{t \in (0, \infty)}$ should not be trivial. Indeed, distinguishing different classes or types of signals often requires leveraging subtle information in the signals. This information is often stored in the high frequencies, i.e., $\widehat{f}(k)$ for large $k$. Our problem is thus to find a family of representations that are stable to diffeomorphisms, discriminative between different types of signals, which allow one to control the scale of isometric invariance, and do so for data that is supported on a manifold. The wavelet scattering transform of \citet{mallat:scattering2012} achieves goals analogous to the ones presented here, but for Euclidean supported signals. Therefore, we seek to construct a geometric version of the scattering transform, using filters  corresponding to the spectral geometry of  $\M,$ and to show that it has similar properties to its Euclidean counterpart. 

\section{Spectral Integral Operators}
\label{sec: ops}

Similar to traditional ConvNets, the Euclidean scattering transform constructed in \citet{mallat:scattering2012} consists of an alternating cascade of convolutions and nonlinearities. In the manifold setting, it is not immediately clear how to define convolution operators, because translation is not well defined. In this section, we introduce a family of operators on $\Lb^2(\M)$ that are analogous to convolution operators on Euclidean space due to the characterization of Euclidean convolution operators as Fourier multipliers. These operators will then be used in Sec.~\ref{sec: transform} to construct the geometric scattering transform.

For a function $\eta:\Lambda\rightarrow\mathbb{R},$ we define a {\it spectral kernel} $K_\eta$ by \begin{equation} \label{eqn: convolution kernel}
K_\eta(x,y) = \sum_{k \in \N} \eta (\lambda_k) \varphi_k (x) \overline{\varphi}_k (y)
\end{equation}
and refer to the integral operator $T_\eta,$ with kernel $K_\eta,$ as a {\it spectral integral operator}.  
Grouping together the $\varphi_k$ belonging to each eigenspace $E_\lambda,$ we write \begin{equation}\label{Kdecomp}
    K_\eta(x,y)=\sum_{\lambda\in\Lambda}\eta(\lambda) K^{(\lambda)}(x,y),
\end{equation}
where
$ 
 K^{(\lambda)}(x,y) = \sum_{\lambda_k=\lambda} \varphi_k (x) \overline{\varphi}_k (y).
$ 
Using the fact that $\{ \varphi_k \}_{k \in \N}$ is an orthonormal basis for $\Lb^2 (\M)$, a simple calculation shows that 
\begin{equation} \label{eqn: int op norm in onb}
\| T_\eta f \|_2^2 = \sum_{k \in \N} |\eta (\lambda_k)|^2 |\langle f, \varphi_k \rangle|^2\leq \|\eta\|^2_\infty\|f\|^2,
\end{equation}
and so, if $\eta \in \Lb^\infty,$ then $T_\eta$ is a bounded operator of $\Lb^2(\M)$ with operator norm $\|\eta\|_\infty.$ In particular, if $\|\eta\|_\infty\leq1,$ then $T_\eta$ is nonexpansive.
Operators  of this form are analogous  to  convolution operators defined on $\R^d$ since the latter  are diagonalized in the Fourier basis. To further emphasize this connection, we note the following theorem, which shows that spectral integral operators are equivariant with respect to isometries.
\begin{theorem}\label{invariso}
For every spectral integral operator $T_\eta,$ and for every $f\in\Lb^2(\M),$
\begin{equation*}
T_\eta V_\zeta f=V_\zeta T_\eta f, \quad \forall \, \zeta \in \Isom (\M).
\end{equation*}
\end{theorem} 

\begin{proof}
Let $\zeta$ be an isometry and let $\psi_k (x) = \varphi_k (\zeta (x))$. For $\lambda \in \Lambda$ and $x\in\M,$ we define $\vec{\varphi}_{\lambda} (x) \in \C^{m(\lambda)}$ and $\vec{\psi}_{\lambda} (x) \in \C^{m(\lambda)}$ by 
\begin{equation*}
\vec{\varphi}_{\lambda} (x) = ( \varphi_k (x) )_{k : \lambda_k = \lambda}\quad\text{and}\quad 
\vec{\psi}_{\lambda} (x) = ( \psi_k (x) )_{k : \lambda_k = \lambda}.
\end{equation*}
Since $\zeta$ is an isometry, $\{ \varphi_k \}_{k : \lambda_k = \lambda}$ and $\{ \psi_k \}_{k : \lambda_k = \lambda}$ are both orthonormal bases for $E_{\lambda}$. Therefore, there exists an $m(\lambda) \times m (\lambda)$ unitary matrix $A_{\lambda}$ (that does not depend upon $x$) such that $\vec{\psi}_{\lambda} (x) = A_{\lambda} \vec{\varphi}_{\lambda} (x)$. Using this fact, we see 
\begin{align*}
K^{(\lambda)} \left(\zeta(x), \zeta(y)\right) &= \sum_{k : \lambda_k = \lambda} \varphi_k \left(\zeta(x)\right) \overline{\varphi}_k \left(\zeta(y)\right) 
= \sum_{k : \lambda_k = \lambda} \psi_k (x) \overline{\psi}_k (y) 
= \langle \vec{\psi}_{\lambda} (x), \vec{\psi}_{\lambda} (y) \rangle \\
&= \langle A_{\lambda} \vec{\varphi}_{\lambda} (x), A_{\lambda} \vec{\varphi}_{\lambda} (y) \rangle = \langle \vec{\varphi}_{\lambda} (x), \vec{\varphi}_{\lambda} (y) \rangle = K^{(\lambda)} (x,y).
\end{align*}
Therefore, by (\ref{Kdecomp}), we see  that $K_\eta(\zeta (x), \zeta (y)) = K_\eta(x,y)$ for all $(x,y) \in \M \times \M.$ Now letting $x = \zeta (\widetilde{x})$ and changing variables $y=\zeta(z),$ we have
\begin{align*}
T_\eta V_{\zeta} f(x) &= \int_{\M} K_\eta(x,y) f \left(\zeta^{-1}(y)\right) \, dV(y)
= \int_{\M} K_\eta\left(x, \zeta (z)\right) f(z) \, dV(z) \\
&= \int_{\M} K_\eta\left(\zeta (\widetilde{x}\right), \zeta (z)) f(z) \, dV(z)
= \int_{\M} K_\eta\left(\widetilde{x}, z\right) f(z) \, dV(z) \\
&= T_\eta f (\widetilde{x}) = T_\eta f \left(\zeta^{-1} (x)\right) = V_{\zeta} T_\eta f(x)
\end{align*}
as desired.
\end{proof}

We will consider frame analysis operators that are constructed using a countable family of spectral integral operators. We call a spectral function $g : [0, \infty) \rightarrow \R$ a \textit{low-pass filter} if 
\begin{equation}\label{eqn: deflowpass}
|g (\lambda)|\leq g(0)=1, \, \forall \, \lambda \geq 0,\quad\text{and}\quad\lim_{\lambda \rightarrow \infty} g (\lambda) = 0.
\end{equation} 
Similarly, we call $h : [0, \infty) \rightarrow \R$ a \textit{high-pass filter} if 
 \begin{equation*}
     h(0)=0\quad \text{and}\quad \|h\|_\infty\leq1.
 \end{equation*}
Figure \ref{fig: bunny} illustrates a low-pass operator $T_g$, via its impulse responses, on the popular Stanford bunny~\citep{turk2005stanford} manifold.

We will assume that we have  a low-pass filter $g$,  and a family of high-pass filters $\{h_\gamma\}_{\gamma\in\Gamma},$ which satisfy a Littlewood-Paley type condition 
\begin{equation}\label{LPC}
A \leq m(\lambda) \left[ |g(\lambda)|^2 + \sum_{\gamma \in \Gamma} |h_{\gamma} (\lambda)|^2 \right] \leq B, \quad \forall \, \lambda \in \Lambda
\end{equation} 
for some $0<A\leq B.$ A frame analysis operator $\Phi:\Lb^2 (\M)\rightarrow\ellb^2 (\Lb^2 (\M))$ is then defined by
\begin{equation*}
\Phi f = \left\{ T_g f, ~ T_{h_{\gamma}} f : \gamma \in \Gamma \right\},
\end{equation*}
where  $T_g$ and $T_{h_\gamma}$ are the spectral integral operators corresponding to $g$ and $h_\gamma$, respectively. The Littlewood-Paley type condition \eqref{LPC} implies that the filters used to define $\Phi$ evenly cover the frequencies of $\M$. 
The following proposition shows that if $g$ and $\{h_\gamma\}_{\gamma\in\Gamma}$ satisfy (\ref{LPC}), then $\Phi$ has corresponding upper and lower frame bounds.  

\begin{proposition}
\label{LPprop}
Under the Littlewood-Paley condition (\ref{LPC}), $\Phi$ is a bounded operator from $\Lb^2 (\M)$ to   $\ellb^2 (\Lb^2 (\M))$ and 
\begin{equation*}A \| f \|_2^2 \leq \| \Phi f \|_{2,2}^2 \coloneqq \| T_g f \|_2^2 + \sum_{\gamma \in \Gamma} \| T_{h_{\gamma}} f \|_2^2 \leq B \| f \|_2^2, \enspace \forall \, f \in \Lb^2 (\M). 
\end{equation*}
In particular, if $A = B = 1$, then $\Phi$ is an isometry. 

\end{proposition} 
Proposition \ref{LPprop} is proved  by using (\ref{eqn: int op norm in onb}) to write

\begin{align*}
\| \Phi f \|_{2,2}^2 &= \| T_g f \|_2^2 + \sum_{\gamma \in \Gamma} \| T_{h_{\gamma}} f \|_2^2
= \sum_{k \in \N} |g(\lambda_k)|^2 |\langle f, \varphi_k \rangle|^2 + \sum_{\gamma \in \Gamma} \sum_{k \in \N} |h_{\gamma} (\lambda_k)|^2 |\langle f, \varphi_k \rangle|^2. 
\end{align*}
We then group together the terms corresponding to each eigenspace, $E_\lambda,$   apply (\ref{LPC}), and then use Parseval's Identity. Note that for each $\lambda\in\Lambda,$ $|g(\lambda)|^2$ and $|h(\lambda)|^2$ each appear $m(\lambda)$ times in the above sum. It is for this reason that the term $m(\lambda)$ is needed in (\ref{LPC}). For full details of the proof, see Appendix~\ref{sec: pfLPprop}.

\subsection{Geometric Wavelet Transforms} \label{sec: wavelet transforms}

As an important example of the frame analysis operator $\Phi$, we define a geometric wavelet transform in terms of a single 
low-pass filter. Given a  spectral function, $\eta : [0, \infty) \rightarrow \R,$ we define  the dilation at scale $2^j$ by
\begin{equation*}
\eta_j (\lambda) = \eta (2^j \lambda),
\end{equation*}
and a normalized spectral function $\widetilde{\eta}_j : \Lambda \rightarrow \R$ by
\begin{equation} \label{eqn: normalized spectral function}
\widetilde{\eta}_j (\lambda) = \frac{\eta_j (\lambda)}{\sqrt{m(\lambda)}}.
\end{equation}
The normalization factor of $\frac{1}{\sqrt{m(\lambda)}}$ will ensure that the wavelet frame constructed below satisfies the Littlewood-Paley condition (\ref{LPC}).

\begin{figure}
\centering
\subfigure[]{
\begin{adjustbox}{width=0.25\linewidth}
\input{bunny_tikz.tex}
\end{adjustbox}
\label{fig: bunny}
}
\subfigure[]{
\includegraphics[trim={0 48pt 0 45pt},width=0.35\linewidth,clip]{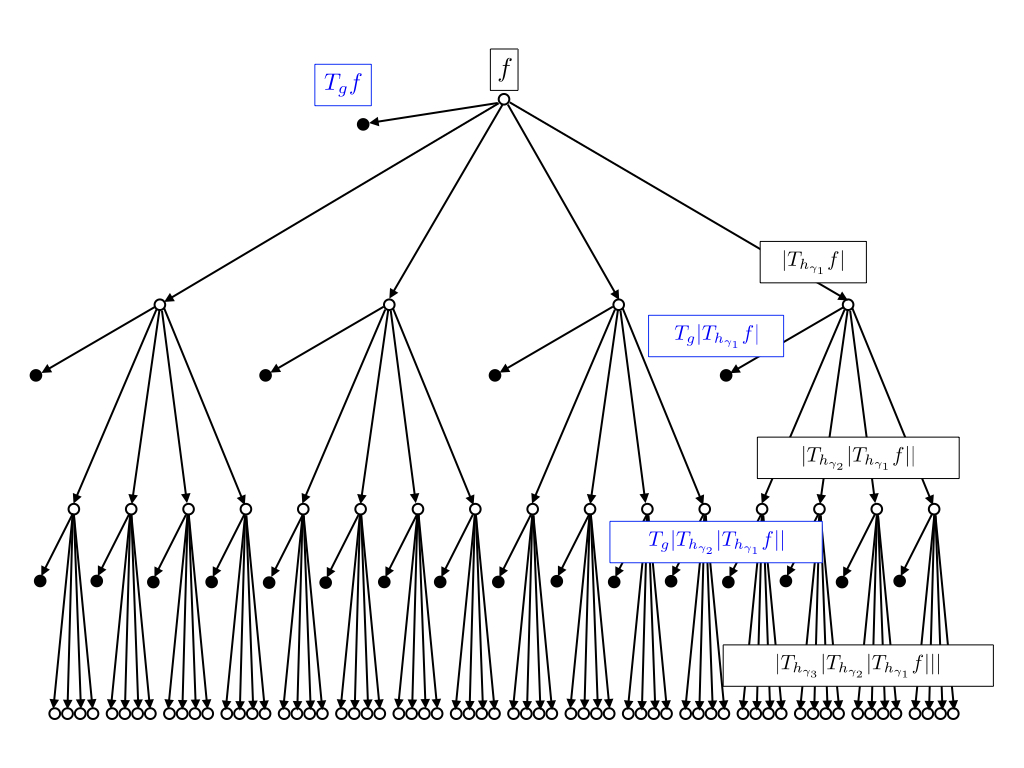}
\label{fig: scattering transform}
}
\caption{\protect\subref{fig: bunny} Illustration of $T_{g} \delta_x$ over the Stanford bunny manifold~\citep{turk2005stanford}, with $g$ a Gaussian restricted to $[0, \infty)$ at scale $2^J$ for three $\log_2$-scales $J$, and three Diracs $\delta_x$ centered at different regions of the manifold. \protect\subref{fig: scattering transform} The geometric scattering transform. Black: Equivariant intermediate layers. Blue: Invariant output coefficients at each layer.}
\end{figure}

Let $g:[0,\infty)\rightarrow\mathbb{R},$ be a nonincreasing low-pass filter, and let $J\in\mathbb{Z}.$ In a manner analogous to classical wavelet constructions, \citep[see for example,][]{meyer:waveletsOperators1993}, we define a high-pass filter $h:[0,\infty)\rightarrow\mathbb{R}$ by
\begin{equation} \label{eqn: wavelet spectral function}
h (\lambda) = \left[ |g(\lambda / 2)|^2 - |g (\lambda)|^2 \right]^{\frac{1}{2}}
\end{equation}
and observe that
\begin{equation} \label{eqn: hjtelescope}
     h_j (\lambda) = \left[ |g_{j-1} (\lambda)|^2 - |g_j (\lambda)|^2 \right]^{\frac{1}{2}}.
\end{equation}
To emphasize the connection with classical wavelet transforms, we denote $T_{\tilde{g}_J}=A_J$ and $T_{\tilde{h}_j}=W_j.$
The geometric wavelet transform $\W_J : \Lb^2 (\M) \rightarrow \ellb^2 (\Lb^2 (\M))$ is then defined as 
\begin{equation*}
\W_J f = \{ A_J f, ~ W_j f \}_{j \leq J} \, .
\end{equation*}


The following proposition can be proved  by observing that  
\begin{equation*}
m (\lambda) \left[ \left|\widetilde{g}_J(\lambda)\right|^2 + \sum_{j \leq J} \left|\widetilde{h}_j (\lambda)\right|^2 \right] =\left|g\left(2^J \lambda\right)\right|^2 + \sum_{j \leq J} \left[ \left|g\left(2^{j-1} \lambda\right)\right|^2 - \left|g\left(2^j \lambda\right)\right|^2 \right] 
\end{equation*}
forms a telescoping sum, and that therefore, $\{A_J,W_j\}_{j\leq J}$ satisfies the Littlewood-Paley condition $(\ref{LPC})$ with $A=B=1.$ We give a complete proof in Appendix~\ref{sec: pfwaveletisometry}.

\begin{proposition}\label{prop: waveletisometry} For all $J\in\mathbb{Z},$
the geometric wavelet transform $\W_J$ is an isometry from $\Lb^2(\M)$ to $\ellb^2(\Lb^2(\M))$, i.e.,  
\begin{equation*}
    \| \W_J f \|_{2,2} = \| f \|_2, \quad \forall \, f \in \Lb^2(\M).
\end{equation*}\end{proposition}

An important example  is  $g (\lambda) = e^{-\lambda}$. In this case the low-pass kernel $G_J (x,y)$ 
corresponding to $T_{g_J}$ is the heat kernel on $\M$ at time $t = 2^J$,  the  kernel $\widetilde{G}_J (x,y)$ corresponding to $A_J=T_{\tilde{g}_J}$ 
is a normalized version, and the corresponding wavelet operators $W_j$ are similar to the diffusion wavelets introduced in \citet{coifman:diffWavelets2006}. The key difference between the construction presented here and the original diffusion wavelet construction is the normalization of the spectral function according to the multiplicity of each eigenvalue, which is needed to obtain an isometric wavelet transform.

\section{The Geometric Scattering Transform}
\label{sec: transform}

The {\it geometric scattering transform} is a nonlinear operator $S:\Lb^2(\M)\rightarrow \ellb^2 (\Lb^2 (\M))$ constructed through an alternating cascade of spectral integral operators and nonlinearities.
Let $M: \Lb^2 (\M) \rightarrow \Lb^2 (\M)$ be the modulus operator,
$Mf(x) = |f(x)|,
$
and for each $\gamma\in\Gamma,$ we let 
\begin{equation*}
U_{\gamma} f(x) = M T_{h_{\gamma}} f(x) = |T_{h_{\gamma}} f(x)|. 
\end{equation*}
For $m\geq 1,$ let  $\Gamma_m$ be the set of all  paths of the form
$\vec\gamma=(\gamma_1, \ldots, \gamma_m),$ and 
let $\Gamma_\infty=\cup_{m=1}^\infty \Gamma_m$ denote the set of a all finite paths.
For $\vec{\gamma} \in \Gamma_m,$ $m\geq 1$ let

\begin{equation*}
U_{\vec{\gamma}} f(x) = U_{\gamma_m} \cdots U_{\gamma_1} f(x), \quad \vec{\gamma} = (\gamma_1, \ldots, \gamma_m).
\end{equation*}
We define an operator $U:\Lb^2(\M)\rightarrow \ellb^2 (\Lb^2 (\M))$, called the scattering propagator, by 
\begin{equation*}
U f = \{ U_{\vec{\gamma}}f : \vec{\gamma} \in \Gamma_\infty \}.
\end{equation*}
In a slight abuse of notation, we include the ``empty-path'' of length zero in $\Gamma_\infty,$ and define $U_\gamma$ to be the identity if $\gamma$ is the empty path.

The scattering transform $S_{\vec{\gamma}}$ over a path $\vec{\gamma} \in \Gamma_{\infty}$ is defined as the integration of $U_\gamma$ against the low-pass integral operator $T_g,$ i.e.,
$ 
S_{\vec{\gamma}} f(x) = T_g U_{\vec{\gamma}} f(x)
$ 
. The operator $S :\Lb^2(\M)\rightarrow \ellb^2 (\Lb^2 (\M))$ given by
$$ 
S f = \{ S_{\vec{\gamma}}f : \vec{\gamma} \in \Gamma_\infty \},
$$ 
is referred to as the geometric scattering transform. It is a mathematical model for geometric ConvNets that are constructed out of spectral filters. The following proposition shows that $S$ is nonexpansive if the Littlewood-Paley condition holds with $B=1$, which we assume for the rest of the paper. The proof is nearly identical to Proposition 2.5 of \citet{mallat:scattering2012}, and is thus omitted. 

\begin{proposition}\label{nonexpansive}
If the Littlewood-Paley condition \eqref{LPC} holds with $B=1$, then
\begin{equation*}
\| S f_1 - S f_2 \|_{2,2} \leq \| f_1 - f_2 \|, \quad \forall \, f_1, f_2 \in \Lb^2 (\M).
\end{equation*}
\end{proposition}

The scattering transform is invariant to the action of the isometry group on the inputted signal $f$ up to a factor that depends upon the decay of the low-pass spectral function $g$. If $|g (\lambda)| \leq C_0 e^{-t \lambda}$ for some constant $C_0$ and $t > 0$ (e.g., the heat kernel), then the following theorem establishes isometric invariance up to the scale $t^d$. 

\begin{theorem} \label{thm: scat isometry invariance}
Let $\zeta \in \Isom (\M)$ and  $|g (\lambda)| \leq C_0 e^{-t \lambda}$ for some constant $C_0$ and $t > 0$. Then there exists a constant $C(\M)<\infty,$ such that
\begin{equation*}
\| S  f - S  V_{\zeta} f \|_{2,2} \leq C_0 C(\M) t^{-d} \| \zeta \|_{\infty} \| U  f \|_{2,2}, \quad \forall \, f \in L^2(\M). 
\end{equation*}
\end{theorem}

\begin{proof}
We rewrite 
$ 
S f - S  V_{\zeta}f = [V_{\zeta}, S  ]f + S f - V_{\zeta} S f. 
$ 
Theorem \ref{invariso} proves that spectral integral operators commute with isometries. Since the modulus operator does as well, it follows that  
$[V_{\zeta}, S  ]f = 0$ and 
\begin{equation*}
\| S f - S  V_{\zeta}f \|_{2,2} = \| S f - V_{\zeta} S f \|_{2,2}.
\end{equation*}
Since $S  = T_g U $, we see that
\begin{align}
\| S f - V_{\zeta} S f \|_{2,2} &= \| T_g U f - V_{\zeta} T_g U f \|_{2,2} 
\leq \| T_g - V_{\zeta} T_g \| \| U f \|_{2,2}, \label{eqn: S to U bound}
\end{align}
Since $|g (\lambda)| \leq C_0 e^{-t \lambda},$ Lemma \ref{lem: filter output stability} stated below shows that 
\begin{equation*}
\| T_g - V_{\zeta} T_g \| \leq C_0 C(\M) t^{-d} \| \zeta \|_{\infty},
\end{equation*}
which completes the proof.
\end{proof}

\begin{lemma} \label{lem: filter output stability}
There exists a constant $C(\M) > 0,$ such that
for every spectral integral operator, $T_{\eta},$  and for every $\zeta \in \Diff (\M)$,
\begin{equation*}
\| T_\eta f - V_{\zeta} T_\eta f \| \leq C (\M) \left(\sum_{k \in \N} \eta (\lambda_k) \lambda_k^{d/2}\right) \| \zeta \|_{\infty}\|f\|_2, \quad \forall \, f\in\Lb^2(\M).
\end{equation*}
Moreover, if 
 $|\eta (\lambda)| \leq C_0 e^{-t \lambda}$ for some constant $C_0$ and some $t > 0$, then there exists a constant $C'(\M) > 0$
such that for any $\zeta \in \Diff (\M)$, 
\begin{equation*}
\| T_\eta f - V_{\zeta} T_\eta f \| \leq C_0C'(\M) t^{-d} \| \zeta \|_{\infty}\|f\|_2, \quad \forall \, f\in\Lb^2(\M).
\end{equation*}
\end{lemma}
We give the proof of Lemma \ref{lem: filter output stability} in Appendix~\ref{sec: pffilteroutputstability}.

\section{Stability to Diffeomorphisms}
\label{sec: stability}

In the previous section, we showed that the scattering transform is invariant to the action of isometries up to a factor depending on the scale $t.$
In \citet{mallat:scattering2012}, it is shown that the Euclidean scattering transform 
is stable to the action of certain diffeomorphisms that are close to being translations. In an effort to prove analogous results for the geometric scattering transform, we introduce three quantities $A_1, A_2,$ and $A_3$, which measure how far away a diffeomorphism  is from being an isometry. If $\zeta$ is an isometry, then by definition  $r(\zeta(x),\zeta(y))=r(x,y)$ for all $x$ and $y$ in $\M$. This motivates us to define 
\begin{equation}\label{a1}
A_1 (\zeta) \coloneqq \sup_{\substack{x,y \in \M \\ x \neq y}} \left|\frac{r\left(\zeta (x), \zeta (y)\right) - r(x,y)}{r(x,y)}\right|.
\end{equation}
It is also known \citep[see for example,][]{kobayashi1996foundations} that isometries are volume preserving in the sense that $|\det[D\zeta (x)]|=1$ for all $x\in\M.$ Therefore,  we introduce the quantity 
\begin{equation}\label{a2}
A_2(\zeta) \coloneqq \left(\sup_{x \in \M} \Big||\det[D\zeta (x)]| - 1\Big|\right)\left(\sup_{x\in\M} \Big|\det[D\zeta^{-1}(x)]\Big|\right),
\end{equation}
which measures how much a $\zeta$ distorts volumes. Lastly, 
\begin{equation*}
A_3(\zeta) \coloneqq \inf\{\sup_x\{r(\zeta(x),\zeta_1(x))\}: \zeta_1\in \Isom(\M)\}
\end{equation*}
 measures the distance from $\zeta$ to the isomtery group in the $\sup$ norm. We note that 
\begin{align*}
A_3(\zeta) &= \inf\{\sup_x\{r\left(\zeta(x),\zeta_1(x)\right)\}: \zeta_1\in \Isom(\M)\} \\
&=\inf\{\sup_x\{r\left((\zeta^{-1}_1\circ\zeta)(x),x\right)\}: \zeta_1\in \Isom(\M)\} \\
&= \inf\left\{\|\zeta^{-1}_1\circ\zeta\|_\infty:  \zeta_1\in \Isom(\M)\right\} \\
&= \inf\left\{\|\zeta_2\|_\infty: \zeta=\zeta_1\circ\zeta_2, \, \zeta_1\in \Isom(\M)\right\}.
\end{align*}
Thus,  $A_3(\zeta)$ will be small if $\zeta$ can be factored into a global isometry $\zeta_1$ and a small perturbation~$\zeta_2.$ 

 A key step to showing that the Euclidean scattering transform $S_{\Euc}$ is stable to the action of certain diffeomorphisms is a bound on the commutator norm $\| [S_{\Euc} , V_{\zeta}] f\|_{2,2}.$ This bound is then combined with a bound on
$ 
\| S_{\Euc}  f -  V_{\zeta} S_{\Euc}  f \|_{2,2},
$ 
which is obtained by methods similar to those used to prove Theorem \ref{thm: scat isometry invariance}, and the triangle inequality to produce a bound on $\| S_{\Euc}  f -  V_{\zeta} S_{\Euc}  f \|_{2,2}$. 
This motivates us to study the commutator of spectral integral operators with $V_\zeta$ for diffeomorphisms that are close to being isometries. To this end, in the two theorems below, we show that $ \| [T_{\eta}, V_{\zeta}] \|$ can be controlled in terms of the quantities $A_1$, $A_2$, and $A_3$, as well $\eta$-dependent quantities.

\begin{theorem}\label{thm: A3filterstability}
There exists a constant $C (\M) > 0$ such that if $T_\eta$ is a spectral integral operator and  $\zeta \in \Diff (\M)$, then for all $f\in\Lb^2(\M)$,

\begin{align}\label{singlefilt}
\| [T_{\eta}, V_{\zeta}]f\|_2 
&\leq C (\M) \left[   \left( \sum_{k \in \N} \eta (\lambda_k)^2 \right)^{\frac{1}{2}} A_2(\zeta) +\left(\sum_{k \in \N} \eta (\lambda_k) \lambda_k^{d/2} \right)A_3(\zeta)\right] \|f\|_2 .
\end{align}
\end{theorem}

\noindent On certain manifolds, we can prove a modified version of Theorem \ref{thm: A3filterstability} 
that replaces the term $A_3(\zeta)$ with $A_1(\zeta).$ 

\begin{definition}
A manifold $\M$ is said to be {\it{homogeneous}} (with respect to the action of the isometry group), if for any two  points, $x, y \in \M$, there exists an isometry $\tilde{\zeta}:\M\rightarrow\M$ such that $\tilde{\zeta}(x)=y.$
\end{definition}

\begin{definition}
A manifold $\M$ is said to be {\it{two-point homogeneous}}, if for any two pairs of points, $(x_1,x_2), ~(y_1,y_2)$ such that $r(x_1,x_2)=r(y_1,y_2),$ there exists an isometry $\tilde{\zeta}:\M\rightarrow\M$ such that $\tilde{\zeta}(x_1)=y_1$ and $\tilde{\zeta}(x_2)=y_2.$ 
\end{definition}

\begin{theorem}\label{twoptho}
If $\M$ is two-point homogeneous, then there exists a constant $C(\M)>0$ such that if  $T_{\eta}$ is spectral integral operator and $\zeta \in \Diff (\M)$, then for all $f\in\Lb^2(\M)$,
\begin{align*}\label{singlefilt}
\| [T_{\eta}, V_{\zeta}]f\|_2 
 \leq C(\M) \left[ A_1(\zeta)\left(\sum_{k \in \N} \eta (\lambda_k) \lambda_k^{(d+1)/4}\right)  +A_2(\zeta)\left( \sum_{k \in \N} \eta (\lambda_k)^2 \right)^{\frac{1}{2}}   \right] \|f\|_2 .
\end{align*}
\end{theorem}

The proofs of Theorems \ref{thm: A3filterstability} and \ref{twoptho} are in Appendices \ref{sec: pfA3filterstability} and \ref{sec: pftwoptho}.  The keys to these proofs are   Lemmas \ref{lem: isosmudge kernel commutator}, \ref{lem: dist kernel commutator}, and  \ref{lem: kernel grad Linf bound} below, as well as the fact that spectral kernels are radial on two-point homogeneous manifolds. For proofs of these lemmas, please see Appendices \ref{sec: pfisosmudgekernelcommuatator}, \ref{sec: pfdistkernelcommmutator}, and \ref{sec: pfkernelgradLinfbound}. We note that the assumption that $K$ is radial, made in Lemma \ref{lem: dist kernel commutator}, is equivalent to the assumption (\ref{eqn: isoaction}), made in Lemma \ref{lem: isosmudge kernel commutator}, on two-point homogeneous manifolds, but is a stronger assumption in general.

\begin{lemma} \label{lem: isosmudge kernel commutator}
There exists a constant $C(\M)>0$ such that if $\zeta \in \Diff (\M)$ and  $T$ is a kernel integral operator on $\Lb^2(\M)$ with a kernel that is invariant to the action of isometries, i.e.,
\begin{equation}\label{eqn: isoaction}
    K\left(\tilde{\zeta}(x),\tilde{\zeta}(y)\right)=K(x,y), \quad \forall \, \tilde{\zeta} \in \Isom (\M), 
\end{equation}
then 
\begin{equation*}
\| [T,V_\zeta] f \|_2 \leq C(\M)  \left[   \| K \|_{\Lb^2 (\M \times \M)} A_2(\zeta) +\| \nabla K \|_{\infty} A_3(\zeta)\right] \|f\|_2, \quad \forall \, f\in\Lb^2(\M).
\end{equation*}

\end{lemma}
\begin{lemma} \label{lem: dist kernel commutator}
 There exists a constant $C(\M)>0$ such that if $T$ is a kernel integral operator with a radial kernel  $K(x,y) = \kappa (r(x,y)),$  $\kappa \in \Cb^1 (\R),$ and $\zeta \in \Diff (\M)$, then
\begin{equation*}
\| [T,V_\zeta] f \|_2 \leq C(\M) \left[ \| \nabla K \|_{\infty} A_1(\zeta) +\| K \|_{\Lb^2 (\M \times \M)} A_2(\zeta) \right] \|f\|_2, \quad \forall \, f\in\Lb^2(\M).
\end{equation*}
\end{lemma}

\begin{lemma} \label{lem: kernel grad Linf bound}
There exists a constant $C(\M) > 0$ such that 
$ 
\left\| \nabla K^{(\lambda)} \right\|_{\infty} 
\leq C(\M) m (\lambda) \lambda^{d/2}
$ 
 for all $\lambda\in\Lambda$.
As a consequence, if $K_\eta$ is a spectral kernel, then
\begin{equation*}
\| \nabla K_\eta \|_{\infty} \leq C(\M) \sum_{\lambda \in \Lambda} \eta (\lambda) m(\lambda) \lambda^{d/2} = C (\M) \sum_{k \in \N} \eta (\lambda_k) \lambda_k^{d/2}.
\end{equation*}
Furthermore, if $\M$ is homogeneous, then
$ 
\left\| \nabla K^{(\lambda)} \right\|_{\infty} 
\leq C(\M) m (\lambda) \lambda^{(d+1)/4}
$ 
and 
$
\| \nabla K_\eta \|_{\infty} \leq  C (\M) \sum_{k \in \N} \eta (\lambda_k) \lambda_k^{(d+1)/4}.
$
\end{lemma}

\section{Conclusion}

The generalization of convolutional neural networks to non-Euclidean domains has been the focus of a significant body of work recently, generally referred to as geometric deep learning. In the process, various challenges have arisen, both computational and mathematical, with the fundamental question being whether the success of ConvNets on traditional signals can be replicated on 
more complicated geometric settings such as  graphs and manifolds. We have presented geometric scattering as a mathematical framework for analyzing the theoretical potential of deep convolutional networks in such settings, in a manner  analogous to the approach used by  \citet{mallat:firstScat2010,mallat:scattering2012} in the  Euclidean setting. 
Our construction 
is also closely related to recent attempts at defining finite-graph scattering \citep{gama:diffScatGraphs2018,zou:graphCNNScat2018,gao:graphScat2018}. 
Our construction is defined in a continuous setting and therefore, we view it is a more direct generalization of the original Euclidean scattering transform. Datasets defined on certain graphs can then be viewed as finite subsamples of the underlying manifold, and the graph scattering transform presented in \citet{gao:graphScat2018} can be viewed as a discretized version of the construction presented here.
Finally, our stability results provided here not only establish the potential of deep cascades of \emph{designed} filters to enable rigorous analysis in geoemtric deep learning; they also provide insights into the challenges that arise in non-Euclidean settings, which we expect will motivate further investigation in future work.

\subsubsection*{Acknowledgments}
This research was partially funded by: IVADO (l'institut de valorisation des donn\'{e}es) [\emph{G.W.}]; the Alfred P. Sloan Fellowship (grant FG-2016-6607), the DARPA Young Faculty Award (grant D16AP00117), and NSF grant 1620216~[\emph{M.H.}].


\bibliographystyle{plainnat}
\bibliography{main}

\begin{thebibliography}{39}
\providecommand{\natexlab}[1]{#1}
\providecommand{\url}[1]{\texttt{#1}}
\expandafter\ifx\csname urlstyle\endcsname\relax
  \providecommand{\doi}[1]{doi: #1}\else
  \providecommand{\doi}{doi: \begingroup \urlstyle{rm}\Url}\fi

\bibitem[And\'{e}n and Mallat(2011)]{anden:scatAudioClass2011}
Joakim And\'{e}n and St\'{e}phane Mallat.
\newblock Multiscale scattering for audio classification.
\newblock In \emph{Proceedings of the 12th International Conference on Music
  Information Retrieval (ISMIR 2011)}, pages 657--662, 2011.

\bibitem[And\'{e}n and Mallat(2014)]{anden:deepScatSpectrum2014}
Joakim And\'{e}n and St\'{e}phane Mallat.
\newblock Deep scattering spectrum.
\newblock \emph{IEEE Transactions on Signal Processing}, 62\penalty0
  (16):\penalty0 4114--4128, August 2014.

\bibitem[And\'{e}n et~al.(2018)And\'{e}n, Lostanlen, and
  Mallat]{arXiv:1807.08869}
Joakim And\'{e}n, Vincent Lostanlen, and St\'{e}phane Mallat.
\newblock Classification with joint time-frequency scattering.
\newblock arXiv:1807.08869, 2018.

\bibitem[B\'{e}rard et~al.(1994)B\'{e}rard, Besson, and
  Gallot]{berard:embedManifoldHeatKer1994}
Pierre B\'{e}rard, G\'{e}rard Besson, and Sylvestre Gallot.
\newblock Embedding {R}iemannian manifolds by their heat kernel.
\newblock \emph{Geometric and Functional Analysis}, 4\penalty0 (4):\penalty0
  373--398, 1994.

\bibitem[Bronstein et~al.(2017)Bronstein, Bruna, LeCun, Szlam, and
  Vandergheynst]{Bronstein:geoDeepLearn2017}
Michael~M. Bronstein, Joan Bruna, Yann LeCun, Arthur Szlam, and Pierre
  Vandergheynst.
\newblock Geometric deep learning: Going beyond {E}uclidean data.
\newblock \emph{IEEE Signal Processing Magazine}, 34\penalty0 (4):\penalty0
  18--42, 2017.

\bibitem[Brumwell et~al.(2018)Brumwell, Sinz, Kim, Qi, and
  Hirn]{brumwell:steerableScatLiSi2018}
Xavier Brumwell, Paul Sinz, Kwang~Jin Kim, Yue Qi, and Matthew Hirn.
\newblock Steerable wavelet scattering for 3{D} atomic systems with application
  to {L}i-{S}i energy prediction.
\newblock In \emph{NeurIPS Workshop on Machine Learning for Molecules and
  Materials}, 2018.
\newblock arXiv:1812.02320.

\bibitem[Bruna and Mallat(2011)]{bruna:scatClass2011}
Joan Bruna and St\'{e}phane Mallat.
\newblock Classification with scattering operators.
\newblock In \emph{2011 IEEE Conference on Computer Vision and Pattern
  Recognition (CVPR)}, pages 1561--1566, 2011.

\bibitem[Bruna and Mallat(2013)]{bruna:invariantScatConvNet2013}
Joan Bruna and St\'{e}phane Mallat.
\newblock Invariant scattering convolution networks.
\newblock \emph{IEEE Transactions on Pattern Analysis and Machine
  Intelligence}, 35\penalty0 (8):\penalty0 1872--1886, August 2013.

\bibitem[Bruna et~al.(2014)Bruna, Zaremba, Szlam, and
  LeCun]{bruna:spectralNN2014}
Joan Bruna, Wojciech Zaremba, Arthur Szlam, and Yann LeCun.
\newblock Spectral networks and deep locally connected networks on graphs.
\newblock In \emph{International Conference on Learning Representations
  (ICLR)}, 2014.

\bibitem[Chud{\'a}cek et~al.(2014)Chud{\'a}cek, Talmon, And{\'e}n, Mallat,
  Coifman, Abry, and Doret]{talmon:scatManifoldHeart2014}
V{\'a}clav Chud{\'a}cek, Ronen Talmon, Joakim And{\'e}n, St{\'e}phane Mallat,
  Ronald~R Coifman, Patrice Abry, and Muriel Doret.
\newblock Low dimensional manifold embedding for scattering coefficients of
  intrapartum fetale heart rate variability.
\newblock In \emph{36th Annual International Conference of the IEEE Engineering
  in Medicine and Biology Society}, pages 6373--6376, 2014.

\bibitem[Coifman and Lafon(2006{\natexlab{a}})]{coifman:diffusionMaps2006}
Ronald~R. Coifman and St\'{e}phane Lafon.
\newblock Diffusion maps.
\newblock \emph{Applied and Computational Harmonic Analysis}, 21:\penalty0
  5--30, 2006{\natexlab{a}}.

\bibitem[Coifman and Lafon(2006{\natexlab{b}})]{coifman:geometricHarmonics2006}
Ronald~R. Coifman and St\'{e}phane Lafon.
\newblock Geometric harmonics: {A} novel tool for multiscale out-of-sample
  extension of empirical functions.
\newblock \emph{Applied and Computational Harmonic Analysis}, 21\penalty0
  (1):\penalty0 31--52, July 2006{\natexlab{b}}.

\bibitem[Coifman and Maggioni(2006)]{coifman:diffWavelets2006}
Ronald~R. Coifman and Mauro Maggioni.
\newblock Diffusion wavelets.
\newblock \emph{Applied and Computational Harmonic Analysis}, 21\penalty0
  (1):\penalty0 53--94, 2006.

\bibitem[Defferrard et~al.(2016)Defferrard, Bresson, and
  Vandergheynst]{NIPS2016_6081}
Micha\"{e}l Defferrard, Xavier Bresson, and Pierre Vandergheynst.
\newblock Convolutional neural networks on graphs with fast localized spectral
  filtering.
\newblock In \emph{Advances in Neural Information Processing Systems 29}, pages
  3844--3852, 2016.

\bibitem[Eickenberg et~al.(2017)Eickenberg, Exarchakis, Hirn, and
  Mallat]{eickenberg:3DSolidHarmonicScat2017}
Michael Eickenberg, Georgios Exarchakis, Matthew Hirn, and St\'{e}phane Mallat.
\newblock Solid harmonic wavelet scattering: Predicting quantum molecular
  energy from invariant descriptors of {3D} electronic densities.
\newblock In \emph{Advances in Neural Information Processing Systems 30 (NIPS
  2017)}, pages 6540--6549, 2017.

\bibitem[Eickenberg et~al.(2018)Eickenberg, Exarchakis, Hirn, Mallat, and
  Thiry]{eickenberg:scatMoleculesJCP2018}
Michael Eickenberg, Georgios Exarchakis, Matthew Hirn, St\'{e}phane Mallat, and
  Louis Thiry.
\newblock Solid harmonic wavelet scattering for predictions of molecule
  properties.
\newblock \emph{Journal of Chemical Physics}, 148:\penalty0 241732, 2018.

\bibitem[Gama et~al.(2018)Gama, Ribeiro, and Bruna]{gama:diffScatGraphs2018}
Fernando Gama, Alejandro Ribeiro, and Joan Bruna.
\newblock Diffusion scattering transforms on graphs.
\newblock arXiv:1806.08829, 2018.

\bibitem[Gao et~al.(2018)Gao, Wolf, and Hirn]{gao:graphScat2018}
Feng Gao, Guy Wolf, and Matthew Hirn.
\newblock Graph classification with geometric scattering.
\newblock arXiv:1810.03068, 2018.

\bibitem[Gin{\'e}(1975)]{evarist1975addition}
Evarist Gin{\'e}.
\newblock The addition formula for the eigenfunctions of the {L}aplacian.
\newblock \emph{Advances in Mathematics}, 18\penalty0 (1):\penalty0 102--107,
  1975.

\bibitem[Hirn et~al.(2017)Hirn, Mallat, and
  Poilvert]{hirn:waveletScatQuantum2016}
Matthew Hirn, St\'{e}phane Mallat, and Nicolas Poilvert.
\newblock Wavelet scattering regression of quantum chemical energies.
\newblock \emph{Multiscale Modeling and Simulation}, 15\penalty0 (2):\penalty0
  827--863, 2017.

\bibitem[H\"{o}rmander(1968)]{hormander1968}
Lars H\"{o}rmander.
\newblock The spectral function of an elliptic operator.
\newblock \emph{Acta Mathematica}, 121:\penalty0 193--218, 1968.

\bibitem[Kobayashi and Nomizu(1963)]{kobayashi1996foundations}
Shoshichi Kobayashi and Katsumi Nomizu.
\newblock \emph{Foundations of Differential Geometry, \textup{Volume 1}}.
\newblock Number~15 in Interscience Tracts in Pure and Applied Math. John Wiley
  and Sons Inc., New York, 1963.

\bibitem[Levie et~al.(2017)Levie, Monti, Bresson, and
  Bronstein]{Levie:CayleyNets2017}
Ron Levie, Federico Monti, Xavier Bresson, and Michael~M. Bronstein.
\newblock Cayleynets: Graph convolutional neural networks with complex rational
  spectral filters.
\newblock arXiv:1705.07664, 2017.

\bibitem[Mallat(2010)]{mallat:firstScat2010}
St\'{e}phane Mallat.
\newblock Recursive interferometric representations.
\newblock In \emph{18th European Signal Processing Conference (EUSIPCO-2010)},
  Aalborg, Denmark, 2010.

\bibitem[Mallat(2012)]{mallat:scattering2012}
St{\'e}phane Mallat.
\newblock Group invariant scattering.
\newblock \emph{Communications on Pure and Applied Mathematics}, 65\penalty0
  (10):\penalty0 1331--1398, October 2012.

\bibitem[Meyer(1993)]{meyer:waveletsOperators1993}
Yves Meyer.
\newblock \emph{Wavelets and Operators}, volume~1.
\newblock Cambridge University Press, 1993.

\bibitem[Oyallon and Mallat(2015)]{oyallon:scatObjectClass2014}
Edouard Oyallon and St\'{e}phane Mallat.
\newblock Deep roto-translation scattering for object classification.
\newblock In \emph{2015 IEEE Conference on Computer Vision and Pattern
  Recognition (CVPR)}, pages 2865--2873, 2015.

\bibitem[Shi and Xu(2010)]{shi:gradEigfcnManifold2010}
Yiqian Shi and Bin Xu.
\newblock Gradient estimate of an eigenfunction on a compact {R}iemannian
  manifold without boundary.
\newblock \emph{Annals of Global Analysis and Geometry}, 38:\penalty0 21--26,
  2010.

\bibitem[Shuman et~al.(2013)Shuman, Narang, Frossard, Ortega, and
  Vandergheynst]{shuman:emerging2013}
David~I. Shuman, Sunil~K. Narang, Pascal Frossard, Antonio Ortega, and Pierre
  Vandergheynst.
\newblock The emerging field of signal processing on graphs: Extending
  high-dimensional data analysis to networks and other irregular domains.
\newblock \emph{IEEE Signal Processing Magazine}, 30\penalty0 (3):\penalty0
  83--98, 2013.

\bibitem[Sifre and Mallat(2012)]{sifre:rotoScatTexture2012}
Laurent Sifre and St\'{e}phane Mallat.
\newblock Combined scattering for rotation invariant texture analysis.
\newblock In \emph{Proceedings of the 20th European Symposium on Artificial
  Neural Networks (ESANN 2012)}, 2012.

\bibitem[Sifre and Mallat(2013)]{mallat:rotoScat2013}
Laurent Sifre and St\'{e}phane Mallat.
\newblock Rotation, scaling and deformation invariant scattering for texture
  discrimination.
\newblock In \emph{The IEEE Conference on Computer Vision and Pattern
  Recognition (CVPR)}, June 2013.

\bibitem[Sifre and Mallat(2014)]{mallat:rigidMotionScat2014}
Laurent Sifre and St\'{e}phane Mallat.
\newblock Rigid-motion scattering for texture classification.
\newblock arXiv:1403.1687, 2014.

\bibitem[Tenenbaum et~al.(2000)Tenenbaum, de~Silva, and
  Langford]{tenenbaum:isomap2000}
Joshua~B. Tenenbaum, Vin de~Silva, and John~C. Langford.
\newblock A global geometric framework for nonlinear dimensionality reduction.
\newblock \emph{Science}, 290\penalty0 (5500):\penalty0 2319--2323, 2000.

\bibitem[Turk and Levoy(2005)]{turk2005stanford}
Greg Turk and Marc Levoy.
\newblock The {S}tanford bunny, 2005.

\bibitem[van~der Maaten and Hinton(2008)]{maaten:tSNE2008}
Laurens van~der Maaten and Geoffrey Hinton.
\newblock Visualizing high-dimensional data using t-{SNE}.
\newblock \emph{Journal of Machine Learning Research}, 9:\penalty0 2579--2605,
  2008.

\bibitem[Wolf et~al.(2014)Wolf, Mallat, and Shamma]{wolf:BSS-mlsp}
Guy Wolf, Stephane Mallat, and Shihab~A. Shamma.
\newblock Audio source separation with time-frequency velocities.
\newblock In \emph{2014 IEEE International Workshop on Machine Learning for
  Signal Processing (MLSP)}, Reims, France, 2014.

\bibitem[Wolf et~al.(2015)Wolf, Mallat, and Shamma]{wolf:BSS}
Guy Wolf, Stephane Mallat, and Shihab~A. Shamma.
\newblock Rigid motion model for audio source separation.
\newblock \emph{IEEE Transactions on Signal Processing}, 64\penalty0
  (7):\penalty0 1822--1831, 2015.

\bibitem[Yi et~al.(2017)Yi, Su, Guo, and Guibas]{yi:syncspecCNN2017}
Li~Yi, Hao Su, Xingwen Guo, and Leonidas Guibas.
\newblock Syncspeccnn: Synchronized spectral cnn for 3d shape segmentation.
\newblock In \emph{The Conference on Computer Vision and Pattern Recognition
  (CVPR)}, 2017.

\bibitem[Zou and Lerman(2018)]{zou:graphCNNScat2018}
Dongmian Zou and Gilad Lerman.
\newblock Graph convolutional neural networks via scattering.
\newblock arXiv:1804:00099, 2018.

\end{thebibliography}

\appendix

\section{Proof of Proposition \ref{LPprop}} \label{sec: pfLPprop}
We note that
\begin{equation*}
\| \pi_{\lambda} f \|_2^2 = \sum_{k \, : \, \lambda_k = \lambda} |\langle f, \varphi_k \rangle|^2,
\end{equation*}
where $\pi_\lambda$ denotes projection onto the eigenspace $E_\lambda.$
Therefore, applying \eqref{eqn: int op norm in onb} to  $\|T_gf\|_2^2$ and $\|T_{h_\gamma}f\|_2^2,$ we see that
\begin{align*}
\| \Phi f \|_{2,2}^2 &= \| T_g f \|_2^2 + \sum_{\gamma \in \Gamma} \| T_{h_{\gamma}} f \|_2^2 \\
&= \sum_{k \in \N} |g(\lambda_k)|^2 |\langle f, \varphi_k \rangle|^2 + \sum_{\gamma \in \Gamma} \sum_{k \in \N} |h_{\gamma} (\lambda_k)|^2 |\langle f, \varphi_k \rangle|^2 \\
&= \sum_{k \in \N} \left[ |g(\lambda_k)|^2 + \sum_{\gamma \in \Gamma} |h_{\gamma} (\lambda_k)|^2 \right] |\langle f, \varphi_k \rangle|^2  \\
&= \sum_{\lambda \in \Lambda} m(\lambda) \left[ |g(\lambda)|^2 + \sum_{\gamma \in \Gamma} |h_{\gamma} (\lambda)|^2 \right] \sum_{k \, : \, \lambda_k = \lambda} |\langle f, \varphi_k \rangle|^2 \\
&= \sum_{\lambda \in \Lambda} m(\lambda) \left[ |g(\lambda)|^2 + \sum_{\gamma \in \Gamma} |h_{\gamma} (\lambda)|^2 \right] \| \pi_{\lambda} f \|_2^2.
\end{align*}
Therefore, the result follows from Parseval's Identity and the assumption that
\begin{equation*}
A \| f \|_2^2 \leq \| \Phi f \|_{2,2}^2 = \| T_g f \|_2^2 + \sum_{\gamma \in \Gamma} \| T_{h_{\gamma}} f \|_2^2 \leq B \| f \|_2^2. \eqno\text{\qed}
\end{equation*}

\section{Proof of Proposition \ref{prop: waveletisometry}} \label{sec: pfwaveletisometry}
We will show that
\begin{equation} \label{eqn: wavelet isometry proof 01}
m (\lambda) \left[ |\widetilde{g}_J(\lambda)|^2 + \sum_{j \leq J} |\widetilde{h}_j (\lambda)|^2 \right] = 1, \quad \forall \, \lambda \in \Lambda.
\end{equation}
The result will then follow from Proposition \ref{LPprop}.

We recall from  (\ref{eqn: hjtelescope}) that 
\begin{equation*}
    h_j (\lambda) = \left[ |g_{j-1} (\lambda)|^2 - |g_j (\lambda)|^2 \right]^{\frac{1}{2}},
\end{equation*}
where $g$ is a low-pass filter assumed by 
 (\ref{eqn: deflowpass}) to satisfy
\begin{equation*}
    |g (\lambda)|\leq g(0)=1\text{ for all }\lambda \geq 0\quad\text{and}\quad\lim_{\lambda \rightarrow \infty} g (\lambda) = 0.
\end{equation*}
Therefore, recalling that for any spectral function $\eta,$ the corresponding normalized spectral filter is defined by
$ 
    \tilde{\eta}(\lambda) = \frac{\eta(\lambda)}{\sqrt{m(\lambda)}},
$ 
we see
$ 
m (\lambda) \left[ |\widetilde{g}_J(\lambda)|^2 + \sum_{j \leq J} |\widetilde{h}_j (\lambda)|^2 \right] = \left|g\left(2^J \lambda\right)\right|^2 + \sum_{j \leq J} \left[ \left|g\left(2^{j-1} \lambda\right)|^2 - |g\left(2^j \lambda\right)\right|^2 \right] = \lim_{j \rightarrow -\infty} |g \left(2^j \lambda\right)|^2 = g(0) = 1, \quad \forall \, \lambda \in \Lambda.
$ 
\qed

\section{Proof of Lemma \ref{lem: filter output stability}} \label{sec: pffilteroutputstability}
Let $K_\eta$ be the kernel of $T_\eta$. Then  by the Cauchy-Schwartz inequality and the fact that  $V_\zeta f(x) = f\left(\zeta^{-1}(x)\right),$ 
\begin{align*}
| T_\eta f(x) - V_{\zeta} T_\eta f(x) | &= \left| \int_{\M} \left[K_\eta(x,y) - K_\eta\left(\zeta^{-1}(x),y\right)\right] f(y) \, dV(y) \right| \nonumber \\
&\leq \| f \|_2 \left( \int_{\M} \left|K_\eta(x,y) - K_\eta\left(\zeta^{-1}(x),y\right)\right|^2 \, dV(y) \right)^{1/2} \nonumber \\
&\leq \| f \|_2 \| \nabla K_\eta \|_{\infty}\left( \int_{\M}  \left|r\left(x, \zeta^{-1} (x)\right)\right|^2 \, dV(y) \right)^{1/2} \\
&\leq \| f \|_2 \sqrt{\vol (\M)} \| \nabla K_\eta \|_{\infty} \| \zeta \|_{\infty}. \nonumber
\end{align*}
It follows that  
\begin{equation} \label{eqn: genlemma6}
\| T_\eta f- V_{\zeta} T_\eta f\|_2 \leq \vol (\M) \| \nabla K_\eta \|_{\infty} \| \zeta \|_{\infty} \| f \|_2.
\end{equation}
Lemma \ref{lem: kernel grad Linf bound} shows 
\begin{equation*}
\| \nabla K_\eta \|_{\infty} \leq C(\M) \sum_{k \in \N} \eta (\lambda_k) \lambda_k^{d/2},
\end{equation*}
and therefore
\begin{equation*}
\| T_\eta f- V_{\zeta} T_\eta f\|_2 \leq C(\M) \left( \sum_{k \in \N} \eta (\lambda_k) \lambda_k^{d/2} \right) \| \zeta \|_{\infty} \| f \|_2.
\end{equation*}
Now suppose that $|\eta (\lambda)| \leq C_0 e^{-t \lambda}$ for some constant $C_0$ and $t > 0.$ Theorem 2.4 of  
\citet{berard:embedManifoldHeatKer1994} proves that for any $x \in \M$, $\alpha \geq 0$, and $t > 0$,
\begin{equation*}
\sum_{k \geq 1} \lambda_k^{\alpha} e^{-t \lambda_k} |\varphi_k(x)|^2 \leq C (\M) (\alpha + 1) t^{-(d + 2\alpha)/2}.
\end{equation*}
Integrating both sides over $\M$ yields:
\begin{equation} \label{eqn: weighted heat kernel eig sum}
\sum_{k \geq 1} \lambda_k^{\alpha} e^{-t \lambda_k} \leq C(\M) (\alpha + 1) t^{-(d + 2\alpha)/2}.
\end{equation}
Using the assumption that that $|\eta(\lambda)| \leq C_0e^{-t\lambda},$  (\ref{eqn: genlemma6}),  and \eqref{eqn: weighted heat kernel eig sum} with $\alpha = d/2,$ gives
\begin{equation*}
\| Tf- V_{\zeta} Tf \|_2 \leq C_0C(\M) \left( \sum_{k \geq 1} \lambda_k^{d/2} e^{-t \lambda_k} \right) \| \zeta \|_{\infty} \| f \|_2 \leq C_0C(\M) t^{-d} \| \zeta \|_{\infty} \| f \|_2. \eqno\text{\qed}
\end{equation*}

\section{Proof of Theorem \ref{thm: A3filterstability}}\label{sec: pfA3filterstability}
Let $K_\eta$ be the kernel of $T_\eta.$ As shown in the proof of Theorem \ref{thm: scat isometry invariance},
\begin{equation}
    K_\eta\left(\tilde{\zeta}(x),\tilde{\zeta}(y)\right)=K_\eta(x,y)
\end{equation}
for all $x$ and $y$ in $\M$ and
for all isometries $\tilde\zeta.$ Therefore, 
we may apply Lemma \ref{lem: isosmudge kernel commutator}, to  see that
\begin{equation*}
\| [T_{\eta}, V_{\zeta}] f \|_2 \leq C (\M) \left[ \| \nabla K \|_{\infty} A_3 (\zeta) + \| K \|_{\Lb^2 (\M \times \M)} A_2(\zeta)\right]\|f\|_2.
\end{equation*}
Lemma \ref{lem: kernel grad Linf bound}
implies that
\begin{equation*}
\| \nabla K_{\eta} \|_{\infty} \leq C (\M) \sum_{k \in \N} \eta (\lambda_k) \lambda_k^{d/2},
\end{equation*}
and since $\{\varphi_k\}_{k=0}^\infty$ forms an orthonormal basis for $\Lb^2(\M),$ we may verify that
\begin{equation*}
    \|K_{\eta}\|_{\Lb^2(\M\times\M)} = \left(\sum_{k=0}^\infty |\eta(\lambda_k)|^2\right)^{1/2}.
\end{equation*}
Combining the above results completes the proof. \qed

\section{Proof of Theorem \ref{twoptho}} \label{sec: pftwoptho}
Let $K_\eta$ be the kernel of $T_\eta$. For any two pairs of points, $(x,y)$ and $(x',y')$  such that $r(x, y) = r(x', y'),$ the definition of two-point homogeneity says that there exists an isometry $\tilde{\zeta}$ mapping $x$ to $x'$ and $y$ to $y'$. As shown in the proof of Theorem \ref{invariso}, this implies $K_\eta(x',y') = K_\eta(x,y)$. Therefore, $K_\eta (x,y)$ is radial and we may write $K_\eta (x,y) = \kappa (r(x,y))$ for some $\kappa\in\Cb^1$. 

By definition, any two-point homogeneous manifold is homogeneous, and so applying Lemma \ref{lem: dist kernel commutator}, we see that
\begin{equation*}
\| [T_\eta, V_{\zeta}] f\|_2 \leq C (\M) \left[ \| \nabla K_\eta \|_{\infty} A_1 (\zeta) + \| K_\eta \|_{\Lb^2 (\M \times \M)} A_2(\zeta)\right]\|f\|_2.
\end{equation*}
Lemma \ref{lem: kernel grad Linf bound} implies that
\begin{equation*}
\| \nabla K_\eta \|_{\infty} \leq C (\M) \sum_{k \in \N} \eta (\lambda_k) \lambda_k^{(d+1)/4},
\end{equation*}
and since $\{\varphi_k\}_{k=0}^\infty$ forms an orthonormal basis for $\Lb^2(\M),$ it can be checked that  
\begin{equation*}
     \|K_\eta\|_{\Lb^2(\M\times\M)} = \left(\sum_{k=0}^\infty |\eta(\lambda_k)|^2\right)^{1/2}.
\end{equation*}
The proof follows from combining the above inequalities. \qed

\section{Proof of Lemma \ref{lem: isosmudge kernel commutator}}\label{sec: pfisosmudgekernelcommuatator}
We first compute
\begin{align*}
|[T,V_\zeta]f(x)| &= \left|\int_{\M} K(x,y)f\left(\zeta^{-1}(y)\right) \, dV(y) - \int_{\M} K\left(\zeta^{-1}(x),y\right)f(y) \, dV(y)\right| \\
&=\left|\int_{\M} K(x,\zeta(y))f(y)|\det[D\zeta (y)]| \, dV(y) - \int_{\M} K\left(\zeta^{-1}(x),y\right)f(y) \, dV(y)\right| \\
&=\left| \int_{\M} f(y) \left[ K (x,\zeta(y))|\det[D\zeta (y)]|- K\left(\zeta^{-1}(x),y\right) \right] \, dV(y) \right| \\
&\leq \left| \int_{\M} f(y) K(x,\zeta(y))\left[|\det[D\zeta (y)]|-1 \right] \, dV(y) \right| \\ 
&+ \left| \int_{\M} f(y) \left[ K(x,\zeta(y))- K \left(\zeta^{-1}(x),y\right) \right] \, dV(y) \right| \\
&\leq \||\det[D\zeta(y)]|-1\|_\infty \left| \int_{\M} f(y) K (x,\zeta(y)) \, dV(y) \right| \\ 
&+ \left| \int_{\M} f(y) \left[ K(x,\zeta(y))-K(\zeta^{-1}(x),y) \right] \, dV(y) \right|.
\end{align*}
Therefore, by the Cauchy-Schwartz inequality,
\begin{align*}
\|[T,V_\zeta]f\|_2 \leq \|f\|_2 \Bigg[ &\||\det[D\zeta(y) ]|-1\|_\infty \bigg(\int_{\M} \int_{\M} \left|K(x,\zeta(y))\right|^2 \, dV(y) dV(x) \bigg)^{\frac{1}{2}} \\
&+ \bigg( \int_{\M} \int_{\M} \left| K (x,\zeta(y))- K \left(\zeta^{-1}(x),y\right)\right|^2 \, dV(y) dV(x) \bigg)^{\frac{1}{2}} \Bigg].
\end{align*}
We may bound the first integral by observing
\begin{equation*}
\int_{\M} \int_{\M} \left| K (x,\zeta(y)) \right|^2 \, dV(y) dV(x) \leq \|\det[D\zeta^{-1}(y)]\|^2_\infty \int_{\M} \int_{\M} \left| K(x,y)\right|^2 \, dV(y) dV(x).
\end{equation*}
To bound the second integral, let $\zeta=\zeta_1\circ\zeta_2$ be a factorization of $\zeta$ into an isometry $\zeta_1$ and a diffeomorphism $\zeta_2$ such that $\|\zeta_2\|_\infty\leq 2 A_3(\zeta).$  Then
\begin{align*}
\int_{\M} \int_{\M} &\left|K (x,\zeta(y)) - K \left(\zeta^{-1}(x),y\right)\right|^2 \, dV(y) dV(x) \\
&= \int_{\M} \int_{\M} \left|K (x,\zeta_1(\zeta_2(y))) - K \left(\zeta_2^{-1}\left(\zeta_1^{-1}(x)\right),y\right)\right|^2 \, dV(y) dV(x)\\
&= \int_{\M} \int_{\M} \left|K (\zeta_1^{-1}(x),\zeta_2(y)) - K \left(\zeta_2^{-1}\left(\zeta_1^{-1}(x)\right),y\right)\right|^2 \, dV(y) dV(x)
\end{align*}
where the last equality uses the assumption \eqref{eqn: isoaction}.
Now,
\begin{align*}
 \int_{\M} \int_{\M} & \left|K \left(\zeta_1^{-1}(x),\zeta_2(y)\right) - K \left(\zeta_2^{-1}(\zeta_1^{-1}(x)),y\right)\right|^2 \, dV(y) dV(x)\\
 &\leq 2\int_{\M} \int_{\M} \left|K \left(\zeta_1^{-1}(x),\zeta_2(y)\right) - K \left(\zeta_1^{-1}(x),y\right)\right|^2 \, dV(y) dV(x)\\
 &\quad+ 2\int_{\M} \int_{\M} |K (\zeta_1^{-1}(x),y) - K (\zeta_2^{-1}(\zeta_1^{-1}(x)),y)|^2 \, dV(y) dV(x)\\
 &\leq 4\|\zeta_2\|^2_\infty\|\nabla K\|^2_\infty \vol(\M)^2.
\end{align*}
Combining the above inequalities completes the proof. \qed

\section{Proof of Lemma \ref{lem: dist kernel commutator}}\label{sec: pfdistkernelcommmutator}
We repeat the proof of Lemma \ref{lem: isosmudge kernel commutator}, to see that it suffices to show
\begin{equation*}
\int_{\M} \int_{\M} \left|K (x,\zeta(y)) - K \left(\zeta^{-1}(x),y\right)\right|^2  dV(y) dV(x) \leq \left[\|\nabla K\|_\infty A_1(\zeta) \diam(\M) \vol(\M) \right]^2.
\end{equation*}
By the assumption that $K(x,y)=\kappa(r(x,y))$ is radial, we see 
\begin{align*}
\int_{\M} \int_{\M} \left|K (x,\zeta(y)) \right.&\left. K \left(\zeta^{-1}(x),y\right)\right|^2  dV(y) dV(x) \\
&= \int_{\M} \int_{\M} \left|\kappa (r(x,\zeta(y))) - \kappa\left(r\left(\zeta^{-1}(x),y\right)\right)\right|^2 \, dV(y) dV(x) \\
&\leq \|\kappa'\|_\infty^2 \int_{\M} \int_{\M} \left|r(x,\zeta(y))- r\left(\zeta^{-1}(x),y\right)\right|^2 \, dV(y) dV(x) \\
&\leq \left[ \|\kappa'\|_\infty A_1(\zeta) \right]^2 \int_{\M} \int_{\M} \left| r(x,\zeta(y)) \right|^2 \, dV(y) dV(x) \\
&\leq \left[ \|\kappa'\|_\infty A_1(\zeta) \diam(\M) \vol(\M) \right]^2.
\end{align*}
Since $K(x,y)=\kappa(r(x,y)),$ we see that 
$ 
\|\nabla K\|_\infty = \|\kappa'\|_\infty,
$ 
which completes the proof. \qed

\section{Proof of Lemma~\ref{lem: kernel grad Linf bound}}\label{sec: pfkernelgradLinfbound}
For any $\lambda_k=\lambda,$ it is a consequence of H\"ormander's local Weyl law (\citealp{hormander1968}; see also \citealp{shi:gradEigfcnManifold2010}) that
\begin{align}
\| \varphi_k \|_{\infty} &\leq C(\M) \lambda^{(d-1)/4}. \label{eqn: eigfcn Linf bound}
\end{align}
 Theorem 1 of  \citet{shi:gradEigfcnManifold2010} shows that 
\begin{align}
\| \nabla \varphi_k \|_{\infty} &\leq C(\M) \sqrt{\lambda} \| \varphi_{k} \|_{\infty}. \label{eqn: grad eigfcn Linf bound}
\end{align}
Therefore,
\begin{align*}
\left|\nabla K^{(\lambda)} (x,y)\right|^2 &= \left| \sum_{k : \lambda_k = \lambda} \nabla \varphi_k (x) \overline{\varphi}_k (y) \right|^2 
\leq \left( \sum_{k : \lambda_k = \lambda} |\nabla \varphi_k (x)|^2 \right) \left( \sum_{k : \lambda_k = \lambda} |\varphi_k (y)|^2 \right) \\
&\leq C(\M) m (\lambda) \lambda^{(d-1)/2} \sum_{k : \lambda_k = \lambda} |\nabla \varphi_k (x)|^2 
\leq C (\M) m(\lambda) \lambda^{(d+1)/2} \sum_{k : \lambda_k = \lambda} \| \varphi_k \|_{\infty}^2 \\
&\leq C(\M) m(\lambda)^2 \lambda^d.
\end{align*}
Furthermore, if we assume that $\M$ is homogeneous, then 
Theorem 3.2 of \citet{evarist1975addition} shows that 
$ 
    \sum_{k : \lambda_k = \lambda} |\varphi_k (y)|^2= C(\M)m(\lambda).
$ 
Substituting this into the above string of inequalities yields
\begin{equation*}
\left|\nabla K^{(\lambda)} (x,y)\right|^2 \leq C(\M) m(\lambda)^2 \lambda^{(d+1)/2}.
\end{equation*}
The bounds on $\|K_\eta\|_\infty$ follow from recalling that 
$ 
    K_\eta(x,y) = \sum_{\lambda\in\Lambda}\eta(\lambda)K^{(\lambda)}(x,y)
$ 
and applying the triangle inequality. \qed

\end{document}